\documentclass[conference]{IEEEtran}
\IEEEoverridecommandlockouts
\usepackage{cite}
\usepackage{amsmath,amssymb,amsfonts}
\usepackage{algorithmic}
\usepackage{graphicx}
\usepackage{subfigure}
\usepackage{textcomp}
\usepackage{xcolor}
\usepackage{booktabs}
\usepackage{multirow}
\usepackage{multicol}
\usepackage[ruled,linesnumbered]{algorithm2e}
\usepackage{makecell}
\usepackage{bbding}
\usepackage{colortbl}
\usepackage{diagbox}

\usepackage{amsthm}     
\usepackage{tcolorbox}
\tcbuselibrary{theorems}

\newtheorem{theorem}{Theorem}[subsection]
\newtheorem{definition}{Definition}[subsection]

\tcolorboxenvironment{theorem}{
  boxrule=0.5pt,
  boxsep=0pt,
  left=4pt,
  right=4pt,
  top=3pt,
  bottom=3pt,
  colback=white 
}

\tcolorboxenvironment{definition}{
  boxrule=0.5pt,
  boxsep=0pt,
  left=4pt,
  right=4pt,
  top=3pt,
  bottom=3pt,
  colback=white 
}

\def\BibTeX{{\rm B\kern-.05em{\sc i\kern-.025em b}\kern-.08em
    T\kern-.1667em\lower.7ex\hbox{E}\kern-.125emX}}

\begin{document}

\title{ComPrivDet: Efficient Privacy Object Detection in Compressed Domains Through Inference Reuse}

\author{Yunhao Yao, Zhiqiang Wang, Ruiqi Li, Haoran Cheng, Puhan Luo, Xiangyang Li,~\IEEEmembership{Fellow,~IEEE} \\
University of Science and Technology of China (USTC)}

\maketitle

\begin{abstract}
As the Internet of Things (IoT) becomes deeply embedded in daily life, users are increasingly concerned about privacy leakage, especially from video data.
Since frame-by-frame protection in large-scale video analytics (e.g., smart communities) introduces significant latency, a more efficient solution is to selectively protect frames containing privacy objects (e.g., faces).
Existing object detectors require fully decoded videos or per-frame processing in compressed videos, leading to decoding overhead or reduced accuracy. 
Therefore, we propose ComPrivDet, an efficient method for detecting privacy objects in compressed video by reusing I-frame inference results. 
By identifying the presence of new objects through compressed-domain cues, ComPrivDet either skips P- and B-frame detections or efficiently refines them with a lightweight detector.
ComPrivDet maintains $99.75\%$ accuracy in private face detection and $96.83\%$ in private license plate detection while skipping over $80\%$ of inferences.
It averages $9.84\%$ higher accuracy with $75.95\%$ lower latency than existing compressed-domain detection methods.
\end{abstract}

\begin{IEEEkeywords}
privacy object detection, inference reuse
\end{IEEEkeywords}

\section{Introduction}
\label{sec:intro}

The growth of IoT has enabled many daily applications, such as smart homes and smart communities~\cite{yao2025trafficdiary,yao2023traffic}.
However, these applications also raise significant privacy concerns, especially in video data that often contains personal information such as faces and license plates~\cite{yao2024secoinfer,yuan2023packetgame}.
In large-scale video processing, frame-by-frame protection can cause severe bottlenecks~\cite{yao2024secoinfer, yao2025privguardinfer, yuan2023packetgame}. 
A more efficient approach is to identify frames with privacy-sensitive objects for selective protection.

\textbf{Pixel-domain detectors incur additional decoding overhead.} 
Pixel-domain object detectors have proven effective across numerous applications~\cite{girshick2014rich,girshick2015fast,liu2016ssd,redmon2016you,tan2020efficientdet}, with efficient one-stage models like YOLO~\cite{redmon2016you} widely used for real-time video tasks.
However, since real-time video streams are typically compressed (e.g., H.264) for transmission, pixel-domain methods must fully decode each frame, which adds extra latency.

\textbf{Compressed-domain detectors suffer from information loss.}
Frame-level methods operate in the compressed domain, using features like motion vectors and residuals to avoid decoding and boost efficiency~\cite{jaballah2019fast,chen2021fast,shou2019dmc,ma2019effective,alizadeh2019compressed}.
However, the loss of structural details degrades detection accuracy.
Although Group of Pictures (GOP)-level techniques mitigate this by partially decoding Intra-coded frames (I-frames)~\cite{wu2018compressed,tran2023fast,wang2019fast}, frame-by-frame inference remains inefficient.

Therefore, there is an urgent need for a compressed-domain privacy object detection framework that minimizes decoding and inference overhead while preserving accuracy. 
To develop such a framework, we face the following challenges: \\
\textbf{(1) How to compensate for lost structural information.}
In compressed video, motion vectors encode motion trajectories, while residuals capture uncompensated details~\cite{sullivan2012overview}.
Thus, data in the compressed domain tends to preserve texture details while ignoring global and local image structure, which degrades the accuracy of privacy object detection.
\\
\textbf{(2) How to capture inter-frame correlations.}
Partially decoding I‑frames can mitigate information loss~\cite{wu2018compressed,tran2023fast,wang2019fast}.
Since predicted (P-) and bi-predicted frames (B-frames) do not always reference an I-frame, an efficient method is required to link each P- or B-frame to the appropriate I-frame.
\\
\textbf{(3) How to balance efficiency and accuracy.}
Solely using compressed-domain data improves efficiency but reduces accuracy.
Decoding and reusing I‑frame features can offset accuracy loss, but fusing pixel‑ and compressed‑domain features for each frame remains inefficient.
Selective reuse is needed to maintain high accuracy while preserving efficiency.
\\

To address these challenges, we propose ComPrivDet, a compressed-domain privacy detection framework.
To our knowledge, it is the first to introduce inference reuse for privacy detection acceleration.
Our contributions are as follows:

\ noindent~\textbullet~An efficient privacy object detection framework based on multi-model collaboration and inference reuse. 
We partially decode the few, highly variable I-frames for pixel-domain detection.
Leveraging inference reuse, we skip most low-change P/B-frames to accelerate inference, while applying a lightweight model only to the remaining high-change ones.

\noindent~\textbullet~An incremental inter-frame association method. 
Since we need to reuse I-frame inference for P- and B-frames, we introduce an incremental update algorithm, which accumulates motion vectors and residuals during video transmission. 
This reconstructs the motion-compensation relationship between each P/B-frame and the GOP's initial I-frame.

\noindent~\textbullet~An efficient frame-skip policy.
To mitigate potential accuracy loss when reusing I-frame inference, we design a decision-making mechanism based on inter-frame association.
By leveraging the presence of accumulated motion vectors and the specificity of accumulated residuals, we identify P/B-frames with regions hard to compensate from I-frames and decide whether to reuse I-frame results.

\noindent~\textbullet~
We prototype ComPrivDet and evaluate it on a mixed video dataset combining YouTube Faces~\cite{wolf2011face}, UFPR-ALPR~\cite{laroca2018robust}, and GOT-10k~\cite{huang2019got}.
On average, ComPrivDet achieves $1.36\%$ higher accuracy with $57.23\%$ lower inference latency than existing pixel-domain detectors, and $10.93\%$ higher accuracy with $44.65\%$ lower latency than compressed-domain methods.

\begin{figure*}[t]
    \centering
    \includegraphics[width=1.6\columnwidth]{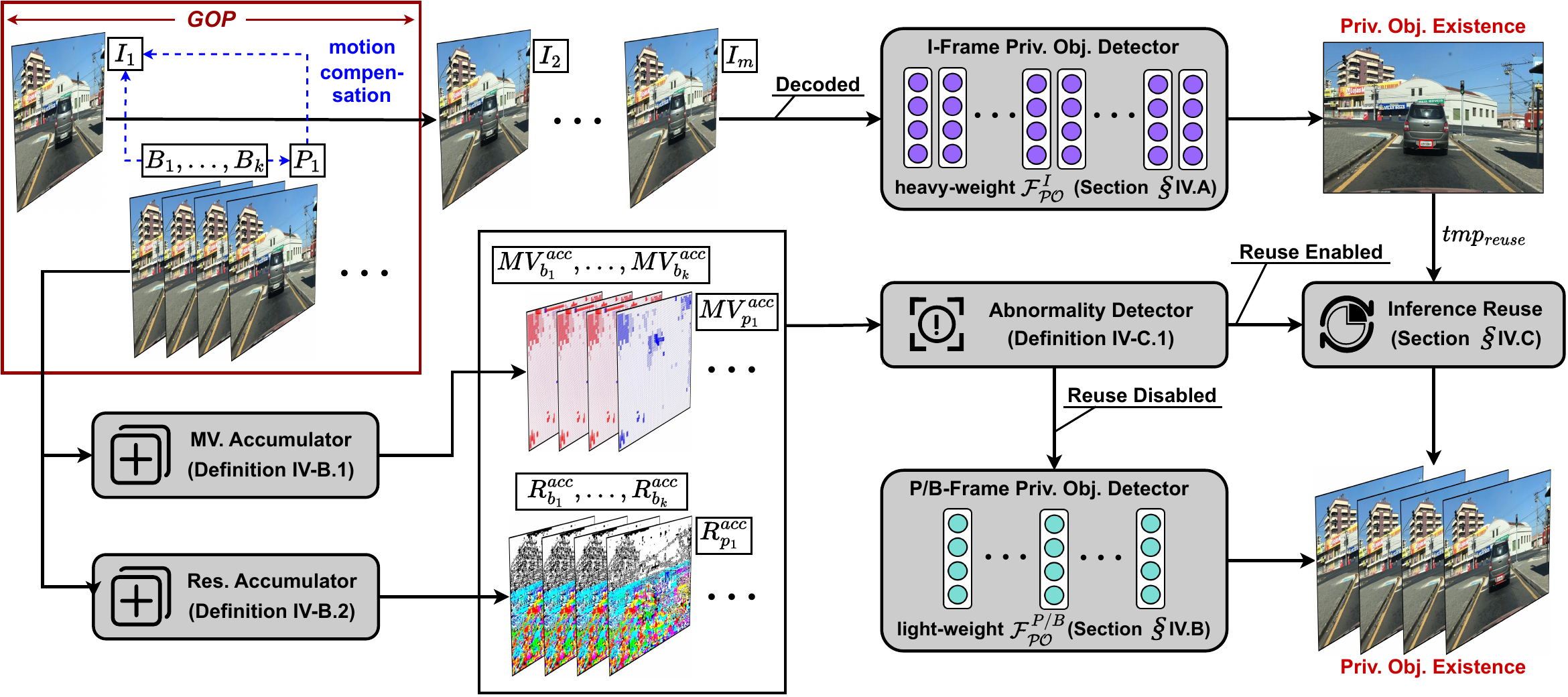}
    \caption{The System Overview of ComPrivDet.}
    \label{fig:system}
\end{figure*}

\section{Related Works}

\textbf{Pixel-Domain Object Detection} operates on fully decoded video frames, leveraging DNNs to extract rich spatial features for localization and classification.
\textit{Two-Stage Detectors} first generate region proposals, followed by feature extraction and classification of each candidate region.
Typical methods include: R-CNN~\cite{girshick2014rich}, which uses Selective Search with an AlexNet backbone; Fast R-CNN~\cite{girshick2015fast}, which extracts a shared feature map once for efficiency; Faster R-CNN~\cite{ren2015faster}, which introduces an RPN to unify proposal generation and detection.
\textit{One-Stage Detectors} directly extracts feature maps from input images with a backbone and simultaneously predicts bounding boxes and class labels.
Typical methods include: SSD~\cite{liu2016ssd}, which adopts VGG-16 as the backbone to extract features at six scales; YOLOv5~\cite{bochkovskiy2020yolov4}, which builds on CSPDarknet-53 and fuses multi-level features via PANet to improve multi-size detection; EfficientDet~\cite{tan2020efficientdet}, which uses EfficientNet and introduces BiFPN for bidirectional feature fusion.

\textbf{Compressed-Domain Object Detection} utilizes information in the video bitstream to detect objects with partially or fully undecoded frames.
\textit{Frame-Level Methods} detect objects on each compressed frame using bitstream cues such as motion vectors, residuals and prediction modes. 
Deep learning-based methods fuse compressed-domain signals into multi-scale feature maps~\cite{jaballah2019fast,chen2021fast} or approximated optical-flow~\cite{shou2019dmc}. Statistical estimation methods construct foreground-background masks to segment object regions~\cite{ma2019effective,alizadeh2019compressed}.
\textit{GOP-Level Methods} extract pixel-domain features from the I-frame using a high-capacity model, and transfer them to a lightweight detector on corresponding predicted frames.
Early work stacks ResNet-152 features onto a ResNet-18 to transfer knowledge~\cite{wu2018compressed}.
MMNet leverages an LSTM to model temporal dependencies and progressively propagate multi-scale features to P-frames~\cite{wang2019fast}.
SftRefNet transfers I-frame predictions to P-frames based on RoIAlign and motion vectors, then refines them with residual corrections~\cite{tran2023fast}.

\section{Problem Setup}
We focus on IoT scenarios that involve large-scale video transmission and processing, such as smart buildings and communities. These scenarios require numerous real-time tasks, including people counting and vehicle detection for property management.
However, users may be unwilling to expose private information, such as faces or license plates, to administrators with insufficient access rights.
Since privacy-sensitive objects appear only in partial video frames, protecting entire videos would result in considerable computational overhead. Therefore, efficiently identifying frames with privacy objects is a critical preprocessing step for privacy protection. 

Let $[F_1, F_2, ..., F_n]$ denote a sequence of consecutive frames captured by a smart camera.
During uploading to the edge server, these frames are compressed by video coding standards such as H.264 and HEVC, producing a compressed video stream $V = [F^{cmp}_1, ..., F^{cmp}_n]$.
To support downstream privacy protection, we require an efficient privacy object detection approach $\mathcal{F}_{\mathcal{PO}}$ to determine whether $F^{cmp}_i$ contains the privacy object $\mathcal{PO}$.
We formally define our objective as follows:

$$
\mathcal{F}_{\mathcal{PO}} = \arg\max_{F_{\mathcal{PO}}} [Acc(F_{\mathcal{PO}}) + \alpha \cdot Lat(F_{\mathcal{PO}})],
$$
$$
Exist_{\mathcal{PO}} = \mathcal{F}_{\mathcal{PO}}([F^{cmp}_1, ..., F^{cmp}_n]),
$$
where $\mathcal{F}_{\mathcal{PO}}$ denotes the optimal detection method balancing detection accuracy $Acc(\cdot)$ and inference latency $Lat(\cdot)$. 
It determines the presence of privacy object $\mathcal{PO}$, $Exist_{\mathcal{PO}}$, in each frame of the compressed video stream $V$.

\section{System Design}
In this section, we provide a detailed explanation of ComPrivDet, which detects privacy objects at GOP-level in compressed video streams.
First, the heavyweight detector $\mathcal{F}^{I}_{\mathcal{PO}}$ identifies privacy objects in decoded I-frames (Section IV.A).
Second, ComPrivDet establishes the correspondence between each P/B-frame and the initial I-frame by incrementally accumulating their motion vectors and residuals (Section IV.B).
Finally, after determining the abnormality of each P- or B-frame, ComPrivDet enhances detection efficiency by either reusing I-frame results (Section IV.C) or activating the lightweight detector $\mathcal{F}^{P/B}_{\mathcal{PO}}$ (Section IV.B).
The system overview of ComPrivDet is presented in Fig.~\ref{fig:system}.

\subsection{I-Frame Detection}
In compressed video streams, the appearance of an I-frame typically marks either the beginning of a periodic GOP or a significant scene change~\cite{sullivan2012overview}. 
In IoT scenarios with fixed cameras, a significant scene change often means a new object has appeared~\cite{yuan2023packetgame}. 
Motivated by this observation, we decode I-frames to obtain richer pixel-domain information for privacy object detection.
Although I‑frames make up only a small portion of the video stream, they are more likely to contain new objects.  
Thus, incurring some extra latency for more accurate detection is a reasonable trade‑off.
Given the lower deployment cost and higher inference efficiency compared to two-stage detectors, one-stage models are better suited to IoT scenarios with stringent real-time requirements~\cite{zou2023object}. 
Accordingly, we adopt YOLOv5~\cite{bochkovskiy2020yolov4} as the backbone of our I-frame privacy object detection model, $\mathcal{F}^{I}_{\mathcal{PO}}(F_i)$, where $F_i$ denotes a decoded I-frame.

\begin{figure}[t]
    \centering
    \includegraphics[width=0.9\columnwidth]{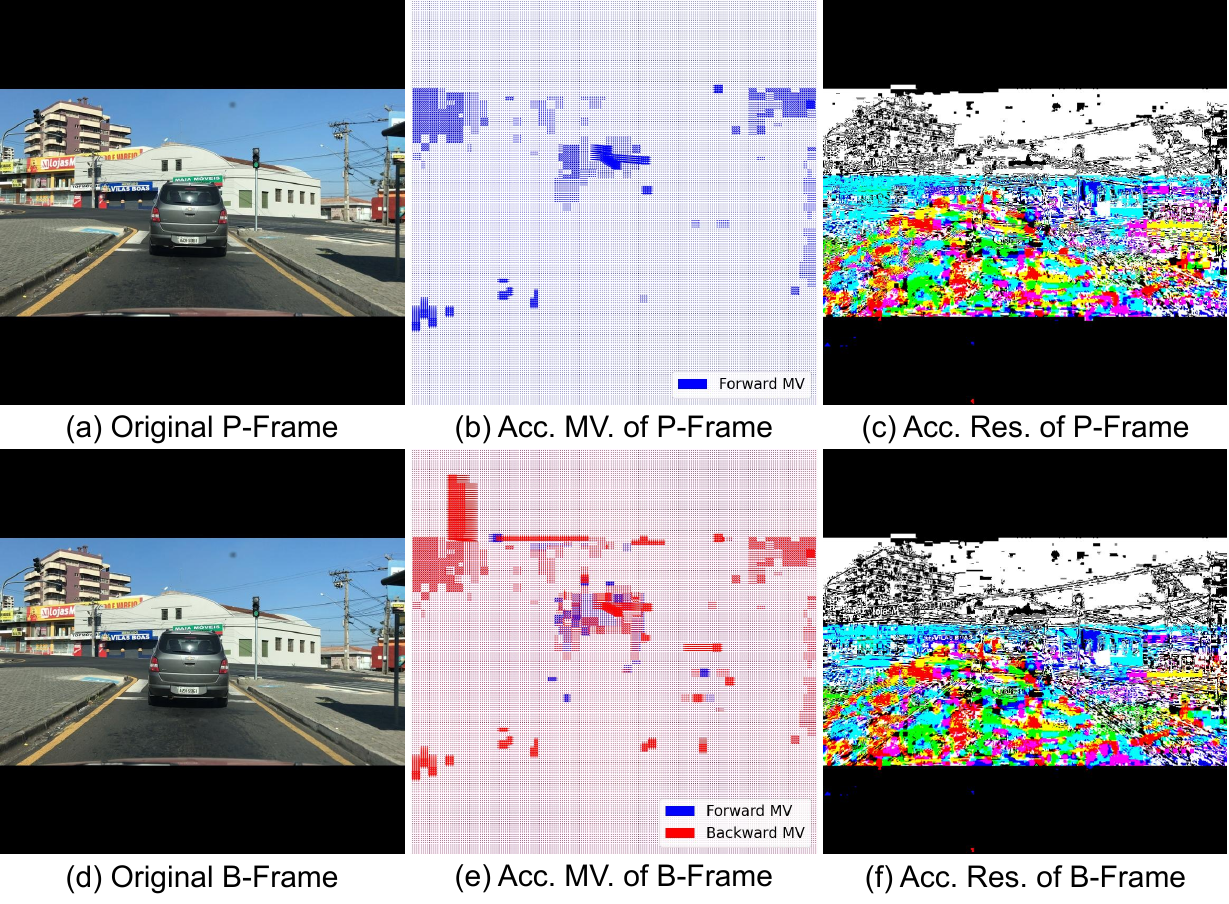}
    \caption{Examples of Accumulated Motion Vectors and Accumulated Residuals.}
    \label{fig:acc_mv_r}
\end{figure}

\subsection{P/B-Frame Detection}
In a compressed video stream, the GOP structure is typically fixed (e.g., I B B B P ...), unless a significant scene change occurs, prompting the insertion of an additional I-frame~\cite{sullivan2012overview}.
Consequently, the likelihood of new objects appearing in P- and B-frames is much lower than in I-frames.
Thus, we refrain from decoding P- and B-frames.
Instead, we leverage compressed-domain data with a lightweight ResNet18~\cite{he2016deep} for accelerated P- and B-frame detection.

In video compression, each prediction block in a P- or B-frame is associated with one or two blocks in reference frames through motion vector(s), while the residuals capture the portion not fully compensated~\cite{sullivan2012overview}.
Since reference frames are typically adjacent I- or P-frames, the original motion vectors and residuals cannot be directly mapped to the I-frame of a GOP.
To enable the reuse of I‑frame results, each P- or B-frame must be associated with the initial I-frame.

For a GOP shaped like $F^{cmp}_1(I),$ $F^{cmp}_2(B),$ $...,$ $F^{cmp}_g(P)$, let $MV^{4 \times \frac{H}{4} \times \frac{W}{4}}_i$ denote the motion vectors pointing to the forward and backward reference frames, and let $R^{3 \times H \times W}_i$ denote the residuals.
Through \textbf{Definitions~\ref{def:acc_mv}} and \textbf{\ref{def:acc_r}}, each P- and B-frame can be linked to the initial I-frame of the GOP through accumulated motion vectors and residuals, as illustrated in Fig.~\ref{fig:acc_mv_r}.
These values can be calculated incrementally during video transmission.
When the I-frame result cannot be reused, privacy object detection for each P/B-frame $F^{cmp}_i$ is performed as $\mathcal{F}^{P/B}_{\mathcal{PO}}(MV^{acc}_i, R^{acc}_i)$.

\begin{definition}[Accumulated Motion Vector $MV^{acc}_i$]
\label{def:acc_mv}
    \quad \\
    (1) For the first P-frame $F^{cmp}_2$, 
    $MV^{acc}_2 = MV_2.$
    
    (2) For any other P-frame $F^{cmp}_i$, assuming that $F^{cmp}_{fwd_i}$ is the nearest P-frame before $F^{cmp}_i$, then:
    $$
    MV^{acc}_i = MV_i + (MV^{acc}_{fwd_i} \circ (Id + (MV_i)_{c \in \{0, 1\}})),
    $$

    (3) For each B-frame $F^{cmp}_i$, assuming that $F^{cmp}_{fwd_i}$ are the nearest I/P-frame before and $F^{cmp}_{bwd_i}$ are the nearest P-frame after $F^{cmp}_i$, then:
    $$
    \begin{aligned}
    MV^{acc}_i = MV_i + (MV^{acc}_{fwd_i} \circ (Id + (MV_i)_{c \in \{0, 1\}})), \\
    MV^{acc}_i = MV_i + (MV^{acc}_{bwd_i} \circ (Id + (MV_i)_{c \in \{2, 3\}})). \\
    \end{aligned}
    $$
    
\end{definition}

\begin{definition}[Accumulated Residual $R^{acc}_i$]
\label{def:acc_r}
    \quad \\
    (1) For the first P-frame $F^{cmp}_2$, 
    $R^{acc}_2 = R_2.$
    
    (2) For any other P-frame $F^{cmp}_i$, assuming that $F^{cmp}_{fwd_i}$ is the nearest P-frame before $F^{cmp}_i$, then:
    $$
    R^{acc}_i = R_i + (R^{acc}_{fwd_i} \circ (Id + (MV_i)_{c \in \{0, 1\}})).
    $$

    (3) For each B-frame $F^{cmp}_i$, assuming that $F^{cmp}_{fwd_i}$ are the nearest I/P-frame before and $F^{cmp}_{bwd_i}$ are the nearest P-frame after $F^{cmp}_i$, then:
    $$
    \begin{aligned}
    R^{acc}_i = R_i & + 0.5 \cdot (R^{acc}_{fwd_i} \circ (Id + (MV_i)_{c \in \{0, 1\}})) \\
    & + 0.5 \cdot (R^{acc}_{bwd_i} \circ (Id + (MV_i)_{c \in \{2, 3\}})), \\
    \end{aligned}
    $$
    where $0.5$ is the default motion compensation weight.
\end{definition}

\subsection{Inference Reuse Mechanism}
In IoT scenarios with fixed-camera settings, the appearance of new objects often triggers new I-frames~\cite{sullivan2012overview}.
Therefore, the probability of new objects appearing in P- and B-frames is typically low, making it practical to reuse I-frame results in subsequent P- and B-frames.
However, a special case occurs when a new privacy object appears but is too small to trigger I-frame insertion, preventing direct reuse of I-frame results.
To address this, we introduce a mechanism that determines whether to reuse the I-frame result or invoke $\mathcal{F}^{P/B}_{\mathcal{PO}}(\cdot)$ on the current P- or B-frame.
Specifically, we define abnormal predicted frames in \textbf{Definition~\ref{def:ab_frame}}.
In this context, blocks with zero motion vectors and residuals exceeding $3\sigma$ flag regions poorly predicted by motion compensation. 
When their number exceeds the threshold $\tau_{ab}$, it typically indicates small new objects.
By checking whether the current P- or B-frame is an abnormal predicted frame, we decide whether to enable $\mathcal{F}^{P/B}_{\mathcal{PO}}(\cdot)$.
Finally, we formalize ComPrivDet’s capability in \textbf{Theorem~\ref{thm:capability_sys}}.

\begin{definition}[Abnormal Predicted Frame]
\label{def:ab_frame}
    We determine whether a P- or B-frame $F^{cmp}_i$ is abnormal based on its association with the initial I-frame as follows:
    $$
    \begin{aligned}
    & T_1 = (MV^{acc}_i == 0), \tilde{R}^{acc}_i = Downsample(R^{acc}_i,4), \\
    & T_2 = (\tilde{R}^{acc}_i < \mu - 3\sigma) | (\tilde{R}^{acc}_i > \mu + 3\sigma), \\
    & Abnormal(F^{cmp}_i) =
    \begin{cases}
    1, & \sum(T_1 \odot T_2) / size(T_1) > \tau_{ab}  \\
    0, & \text{otherwise}
    \end{cases}
    \end{aligned}
    $$
    
\end{definition}

\begin{theorem}[Capability of ComPrivDet]
\label{thm:capability_sys}
    Let $p_I$ and $p_{P/B}$ denote the probabilities of I- and P/B-frame occurrence, $p_{ab}$ the probability of an abnormal P/B-frame, and $p_{new}$ that of a new privacy object.
    We define $C_{I}$ and $C_{P/B}$ as the capability of $\mathcal{F}^I_\mathcal{OP}$ and $\mathcal{F}^{P/B}_\mathcal{OP}$, where capability represents a metric such as accuracy. 
    The overall capability $C_{sys}$ of ComPrivDet can then be expressed as follows:
    $$
    C_{sys} = p_I \cdot C_I + p_{P/B} [(1-p_{ab})(1 - p_{new}) C_I + p_{ab} \cdot C_{P/B}]
    $$
\end{theorem}
\begin{proof}
    The terms $p_I$, $p_{P/B} (1-p_{ab})(1 - p_{new})$, $p_{P/B} \cdot p_{ab}$ correspond to detecting I-frames using $\mathcal{F}^I_{\mathcal{PO}}$, reusing I-frame results when no new privacy objects appear, and detecting P- or B-frames using $\mathcal{F}^{P/B}_{\mathcal{PO}}$, respectively.
\end{proof}


\section{Experiments}

\subsection{Experiment Setup}
We construct our evaluation dataset for ComPrivDet by combining the YouTube Faces Database~\cite{wolf2011face} (YTF), the UFPR-ALPR~\cite{laroca2018robust} dataset, and the Generic Object Tracking Benchmark~\cite{huang2019got} (GOT-10k). 
We define privacy objects as entities that reveal personal information about users. In the three datasets, these correspond to faces and license plates. 
Since UFPR-ALPR is smaller than YTF, we upsample UFPR-ALPR to balance the face and license plate streams. 
We then randomly sample from GOT-10k to construct a balanced dataset with a 1:1:2 ratio of face to license plate to other streams. 
After cropping all videos to $768 \times 768$ pixels, we split them 4:1 into training and test sets, ensuring no copies of the same video stream appear in both sets.
Additionally, all experiments are conducted on an edge server equipped with an Intel(R) Xeon(R) Silver 4210R CPU @ 2.40 GHz and an NVIDIA GeForce RTX 3090 GPU.

\begin{table*}[t]
    \centering
    \caption{Evaluation of Accuracy, Reuse Ratio, Latency (ms) under Different $\tau_{conf}$-$\tau_{ab}$ Settings for Private Face Detection.}
    \resizebox{0.8\linewidth}{!}{
        \begin{tabular}{c|c|c|c|c|c|c|c|c|c|c|c|c|c|c|c}
        \toprule
        \midrule
        $\tau_{ab}$ & \multicolumn{3}{c|}{4e-3} & \multicolumn{3}{c|}{8e-3} & \multicolumn{3}{c|}{1.2e-2} & \multicolumn{3}{c|}{1.6e-2} & \multicolumn{3}{c}{2e-2} \\
        \midrule
        $\tau_{conf}$ & Acc & RR & Lat & Acc & RR & Lat & Acc & RR & Lat & Acc & RR & Lat & Acc & RR & Lat  \\
        \midrule
        0.1 & 0.9569 &  & 63.87 & 0.9661 &  & 53.82 & 0.9811 &  & 40.16 & 0.9894 &  & 27.83 & 0.9975 &  & 23.55 \\
        \cmidrule{1-2}
        \cmidrule{4-5}
        \cmidrule{7-8}
        \cmidrule{10-11}
        \cmidrule{13-14}
        \cmidrule{16-16}
        0.2 & 0.9564 &  & 61.44 & 0.9656 &  & 54.02 & 0.9750 &  & 39.60 & 0.9811 &  & 26.81 & 0.9889 &  & 23.45 \\
        \cmidrule{1-2}
        \cmidrule{4-5}
        \cmidrule{7-8}
        \cmidrule{10-11}
        \cmidrule{13-14}
        \cmidrule{16-16}
        0.3 & 0.9564 & 26.22\% & 59.24 & 0.9656 & 41.58\% & 52.14 & 0.9750 & 65.06\% & 36.48 & 0.9811 & 77.83\% & 24.43 & 0.9889 & 84.75\% & 21.31 \\
        \cmidrule{1-2}
        \cmidrule{4-5}
        \cmidrule{7-8}
        \cmidrule{10-11}
        \cmidrule{13-14}
        \cmidrule{16-16}
        0.4 & 0.9564 &  & 57.98 & 0.9656 &  & 50.56 & 0.9750 &  & 35.63 & 0.9811 &  & 24.35 & 0.9889 &  & 20.38 \\
        \cmidrule{1-2}
        \cmidrule{4-5}
        \cmidrule{7-8}
        \cmidrule{10-11}
        \cmidrule{13-14}
        \cmidrule{16-16}
        0.5 & 0.9561 &  & 56.63 & 0.9644 &  & 46.62 & 0.9689 &  & 32.76 & 0.9731 &  & 24.59 & 0.9806 &  & 19.55 \\
        \midrule
        \bottomrule
        \end{tabular}
    }
    \label{tab:ablation_face}
\end{table*}

\begin{table*}[t]
    \centering
    \caption{Evaluation of Accuracy, Reuse Ratio, Latency (ms) under Different $\tau_{conf}$-$\tau_{ab}$ Settings for Private License Plate Detection.}
    \resizebox{0.8\linewidth}{!}{
        \begin{tabular}{c|c|c|c|c|c|c|c|c|c|c|c|c|c|c|c}
        \toprule
        \midrule
        $\tau_{ab}$ & \multicolumn{3}{c|}{4e-3} & \multicolumn{3}{c|}{8e-3} & \multicolumn{3}{c|}{1.2e-2} & \multicolumn{3}{c|}{1.6e-2} & \multicolumn{3}{c}{2e-2} \\
        \midrule
        $\tau_{conf}$ & Acc & RR & Lat & Acc & RR & Lat & Acc & RR & Lat & Acc & RR & Lat & Acc & RR & Lat  \\
        \midrule
        0.1 & 0.9172 &  & 61.18 & 0.9219 &  & 55.92 & 0.9256 &  & 36.16 & 0.9442 &  & 27.74 & 0.9600 &  & 23.09 \\
        \cmidrule{1-2}
        \cmidrule{4-5}
        \cmidrule{7-8}
        \cmidrule{10-11}
        \cmidrule{13-14}
        \cmidrule{16-16}
        0.2 & 0.9256 &  & 61.16 & 0.9303 &  & 51.77 & 0.9339 &  & 36.53 & 0.9525 &  & 26.99 & 0.9683 &  & 21.62 \\
        \cmidrule{1-2}
        \cmidrule{4-5}
        \cmidrule{7-8}
        \cmidrule{10-11}
        \cmidrule{13-14}
        \cmidrule{16-16}
        0.3 & 0.9250 & 26.22\% & 59.67 & 0.9264 & 41.58\% & 48.29 & 0.9300 & 65.06\% & 35.58 & 0.9442 & 77.83\% & 25.91 & 0.9600 & 84.75\% & 22.28 \\
        \cmidrule{1-2}
        \cmidrule{4-5}
        \cmidrule{7-8}
        \cmidrule{10-11}
        \cmidrule{13-14}
        \cmidrule{16-16}
        0.4 & 0.9167 &  & 58.07 & 0.9181 &  & 48.42 & 0.9217 &  & 35.71 & 0.9358 &  & 25.47 & 0.9517 &  & 19.22 \\
        \cmidrule{1-2}
        \cmidrule{4-5}
        \cmidrule{7-8}
        \cmidrule{10-11}
        \cmidrule{13-14}
        \cmidrule{16-16}
        0.5 & 0.9167 &  & 56.03 & 0.9181 &  & 47.34 & 0.9217 &  & 33.57 & 0.9358 &  & 25.49 & 0.9517 &  & 19.93 \\
        \midrule
        \bottomrule
        \end{tabular}
    }
    \label{tab:ablation_plate}
\end{table*}

\subsection{Ablation Experiments}

\subsubsection{Evaluating Threshold Variations}

In object detection, candidate boxes with confidence scores below a threshold $\tau_{conf}$ are filtered out.
Since privacy object detection is a specific object detection task, and $\tau_{ab}$ controls ComPrivDet’s sensitivity in skipping abnormal P/B-frames, we evaluate its performance under different $\tau_{conf}$-$\tau_{ab}$ settings.
We define the accuracy of privacy object detection as $\frac{1}{N}\sum^N_{i=1} \mathbb{I}(\mathcal{F}_{\mathcal{PO}}(F^{cmp}_i) = y^{gt}_i)$, where $y^{gt}_i$ denotes the ground truth and $\mathbb{I}(\cdot)$ the indicator function.
Our experimental results are presented in TABLE~\ref{tab:ablation_face} and \ref{tab:ablation_plate}.

In both face and license plate detection, increasing $\tau_{ab}$ leads to higher accuracy and lower latency.
Higher $\tau_{ab}$ makes ComPrivDet mark fewer P- and B-frames as abnormal, enabling more I‑frame result reuse and faster detection, reflected by the increased reuse ratio (RR) reported in TABLE~\ref{tab:ablation_face} and \ref{tab:ablation_plate}.
As $\tau_{ab}$ increases, ComPrivDet achieves lower inference latency and higher detection accuracy. 
Although counterintuitive, this can be explained by \textbf{Theorem~\ref{thm:capability_sys}}: The I-frame detector $\mathcal{F}^I_\mathcal{PO}$, which operates on decoded I-frames with richer information, has greater capability $C_I$ than the lightweight P/B-frame model $\mathcal{F}^{P/B}_\mathcal{PO}$ with capability $C_{P/B}$. 
As a larger $\tau_{ab}$ increases the reuse rate of I-frame results, the probability of abnormal P/B-frames, $p_{ab}$, decreases. If the probability $p_{new}$ of a new privacy object appearing is small enough, such that the gain in $(1-p_{ab})(1-p_{new})C_I$ exceeds the loss in $p_{ab} \cdot C_{P/B}$, then the overall capability $C_{sys}$ will improve.

Furthermore, in both face and license plate detection, increasing $\tau_{conf}$ generally reduces ComPrivDet’s latency. This is because a higher $\tau_{conf}$ filters more candidate boxes, thereby lowering analysis overhead.
However, in certain cases, a larger $\tau_{conf}$ leads to longer latency (e.g., $\tau_{ab}=8e^{-3}$, $\tau_{conf}=0.1$ vs. $0.2$ in face detection; $\tau_{ab}=1.2e^{-2}$, $\tau_{conf}=0.3$ vs. $0.4$ in license plate detection). We attribute this anomaly to CPU fluctuations.
Notably, ComPrivDet’s detection accuracy shows no consistent correlation with $\tau_{conf}$. 

\subsubsection{Evaluating Detection Latency Components}

ComPrivDet’s inference reuse mechanism introduces extra time overhead.
Since excessive extra latency limits ComPrivDet’s applicability, we evaluate how much this overhead contributes to the total latency.
Specifically, we fix $\tau_{conf} = 0.1$ for private face detection and $\tau_{conf} = 0.2$ for private license plate detection, corresponding to their optimal settings. Results are shown in Fig.~\ref{fig:lat_ratio_face} and  \ref{fig:lat_ratio_plate}.

In private face detection, abnormal P/B-frame detection adds only $0.59\%$–$2.66\%$ to total inference latency, and even less ($0.36\%$–$1.02\%$) in license plate detection.
This is because, as defined in \textbf{Definition~\ref{def:ab_frame}}, this process involves just a few lightweight computations—three logical and two pooling.
Although ComPrivDet employs efficient YOLOv5 for $\mathcal{F}^I_{\mathcal{PO}}$ and lightweight ResNet18 for $\mathcal{F}^{P/B}_{\mathcal{PO}}$, its inference reuse mechanism adds under $3\%$ additional latency while enabling over $80\%$ frame skipping.
Since the overhead of abnormality detection is negligible, ComPrivDet can be easily extended to more complex detection architectures.

\begin{figure}[t]
  \centering
  \subfigure[Private Face Detection]{
    \includegraphics[width=0.46\linewidth]{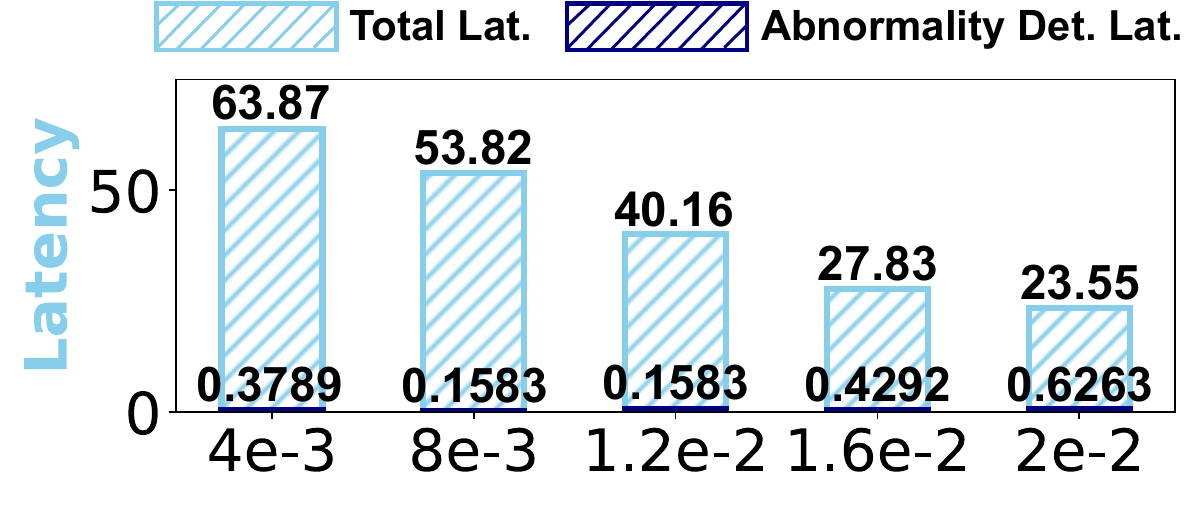}
    \label{fig:lat_ratio_face}
  }
  \subfigure[Private License Plate Detection]{
    \includegraphics[width=0.46\linewidth]{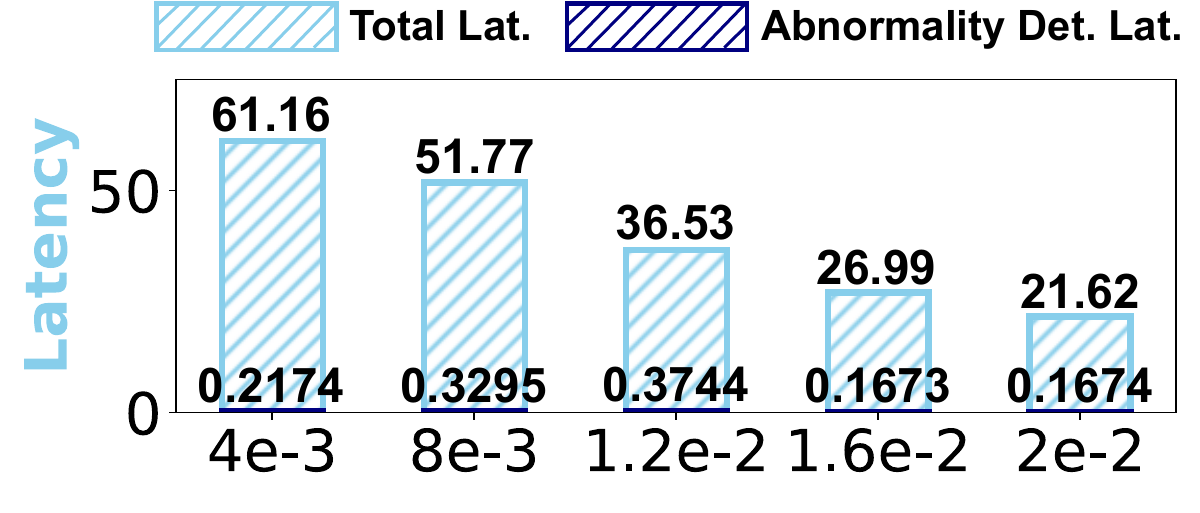}
    \label{fig:lat_ratio_plate}
  }
  \caption{Evaluation of Abnormality Determination Latency Ratio.}
  \label{fig:lat_ratio}
\end{figure}

\begin{figure}[t]
  \centering
  \subfigure[Private Face Detection]{
    \includegraphics[width=0.46\linewidth]{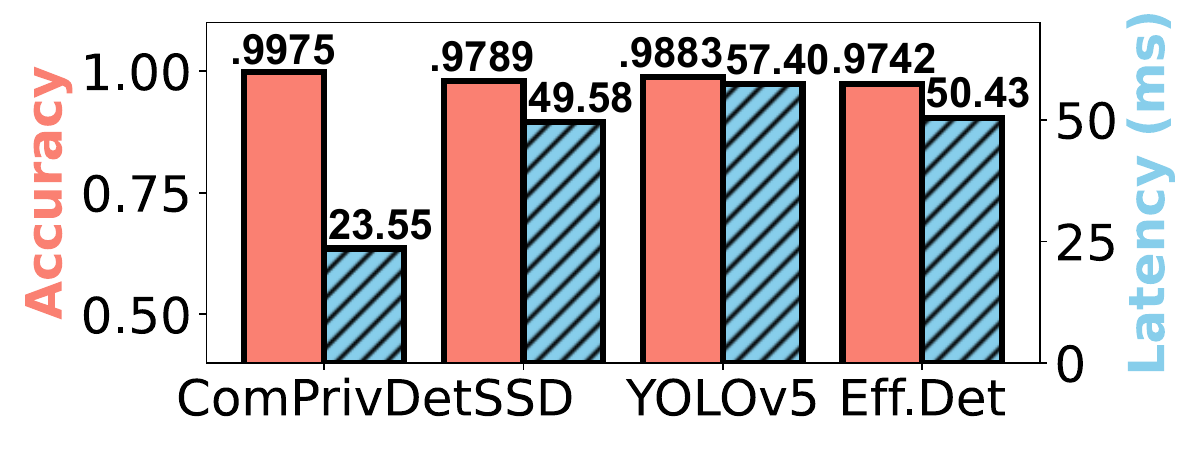}
    \label{fig:cmp_pixel_face}
  }
  \subfigure[Private License Plate Detection]{
    \includegraphics[width=0.46\linewidth]{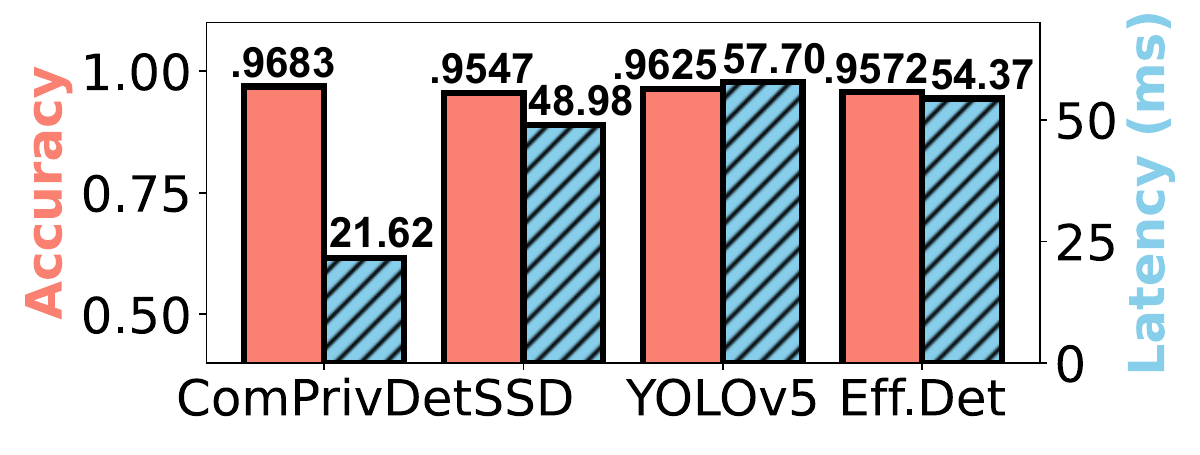}
    \label{fig:cmp_pixel_plate}
  }
  \caption{Comparison with Existing Pixel-Domain Detectors}
  \label{fig:cmp_pixel}
\end{figure}

\subsection{Comparative Experiments}

\subsubsection{Comparison with Existing Pixel-Domain Methods}

One‑stage detectors skip region proposals, offering better efficiency for real‑time IoT scenarios. Thus, we select SSD~\cite{liu2016ssd}, YOLOv5~\cite{bochkovskiy2020yolov4}, and EfficientDet (Eff.Det)~\cite{tan2020efficientdet} as pixel-domain baselines to compare with ComPrivDet on detection capability and efficiency.
Results are presented in Fig.~\ref{fig:cmp_pixel}. 
For private face detection (Fig.~\ref{fig:cmp_pixel_face}), ComPrivDet achieves $0.92\%$-$2.33\%$ higher accuracy, while its inference latency is only $41.03\%$-$47.50\%$ of theirs. 
Similarly, for private license plate detection (Fig.~\ref{fig:cmp_pixel_plate}), ComPrivDet gains $0.58\%$-$1.36\%$ accuracy, with latency reduced to $37.47\%$ to $44.14\%$.

ComPrivDet’s efficiency advantage stems from two factors: it skips detection for most P/B‑frames via inference reuse and only decodes I‑frames.
Surprisingly, ComPrivDet achieves higher accuracy than pixel-domain detectors that perform on fully decoded frames. 
Based on \textbf{Theorem~\ref{thm:capability_sys}} and our manual review, we attribute this to two factors:

\noindent~\textbullet~The lightweight $\mathcal{F}^{P/B}_\mathcal{PO}$ using limited features, has lower capability $C{P/B}$ than the I-frame detector $\mathcal{F}^{I}_\mathcal{PO}$ with $C_I$.
Thus, reusing I-frame results often corrects errors from $\mathcal{F}^{P/B}_\mathcal{PO}$.

\noindent~\textbullet~As compressed-domain cues mainly capture object shapes, they effectively distinguish moving objects from the background in abnormal frames. 
This enables $\mathcal{F}^{P/B}_\mathcal{PO}$ to compensate for errors made by $\mathcal{F}^{I}_\mathcal{PO}$ in some cases.

\begin{figure}[t]
  \centering
  \subfigure[Private Face Detection]{
    \includegraphics[width=0.46\linewidth]{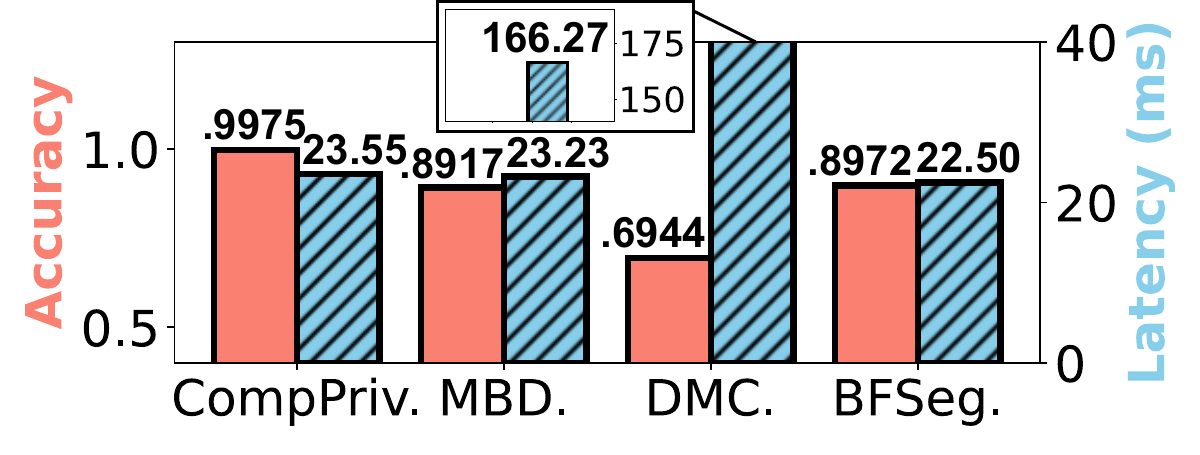}
    \label{fig:cmp_comp_frameLv_face}
  }
  \subfigure[Private License Plate Detection]{
    \includegraphics[width=0.46\linewidth]{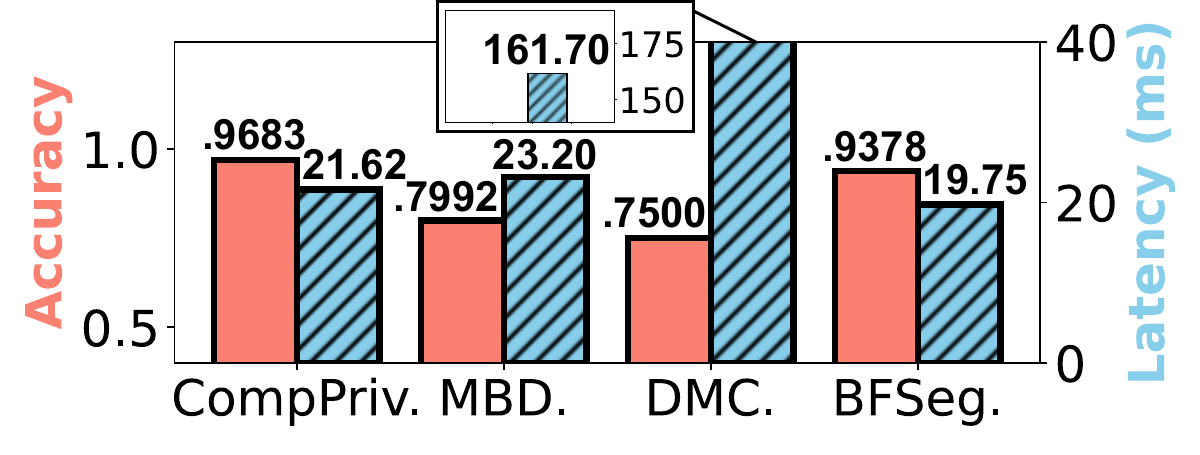}
    \label{fig:cmp_comp_frameLv_plate}
  }
  \caption{Comparison with Existing Compressed-Domain Frame-Level Detectors}
  \label{fig:cmp_comp_frameLv}
\end{figure}

\begin{figure}[t]
  \centering
  \subfigure[Private Face Detection]{
    \includegraphics[width=0.46\linewidth]{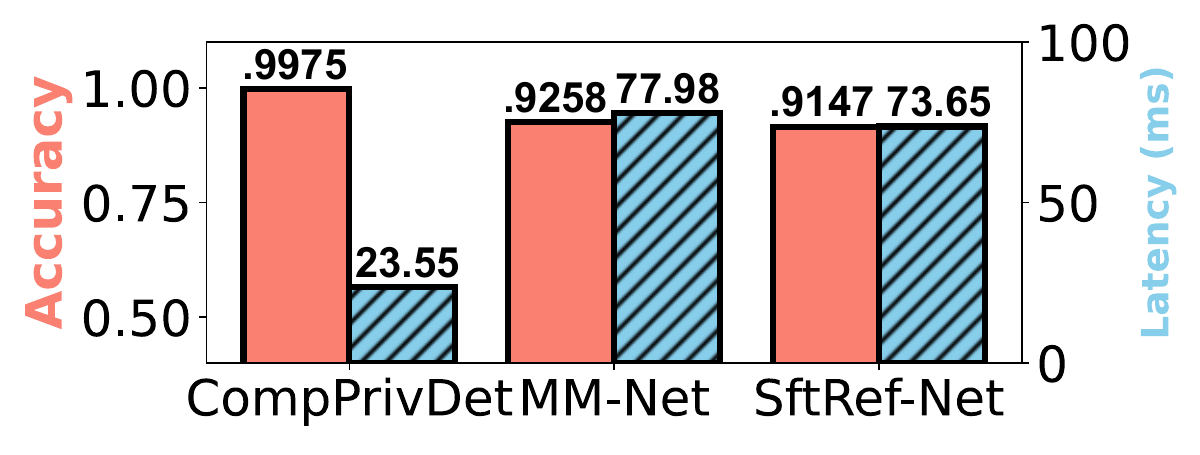}
    \label{fig:cmp_comp_GOPLv_face}
  }
  \subfigure[Private License Plate Detection]{
    \includegraphics[width=0.46\linewidth]{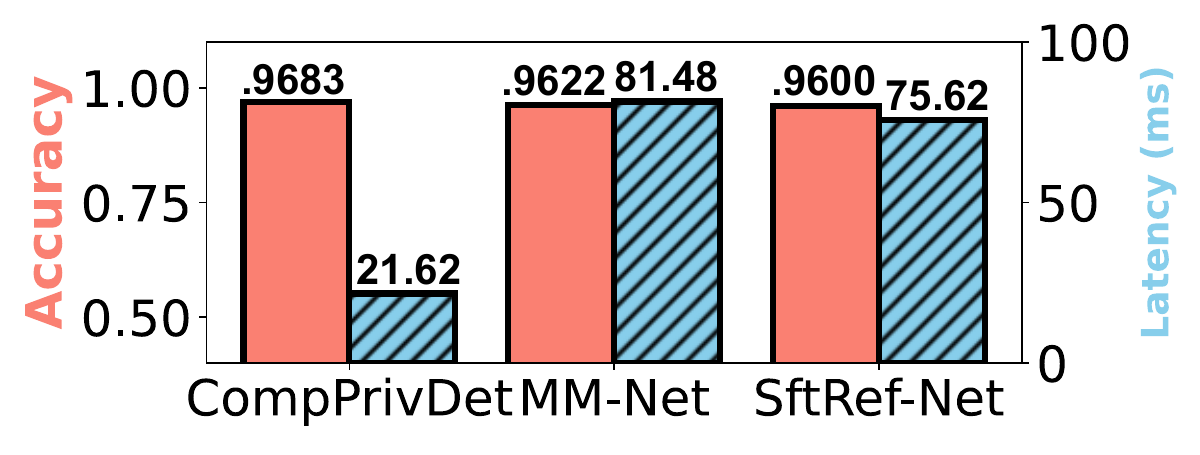}
    \label{fig:cmp_comp_GOPLv_plate}
  }
  \caption{Comparison with Existing Compressed-Domain GOP-Level Detectors}
  \label{fig:cmp_comp_GOPLv}
\end{figure}

\subsubsection{Comparison with Existing Frame-Level Methods}

We compare ComPrivDet with three representative frame‑level compressed‑domain detectors: MCD-Net~\cite{jaballah2019fast} (object masks), DMC-Net~\cite{shou2019dmc} (optical-flow approximation), and BFSeg-Net~\cite{ma2019effective} (statistical estimation).
As shown in Fig.~\ref{fig:cmp_comp_frameLv}, GOP‑level ComPrivDet achieves comparable latency to these methods, differing by only $-6.81\%$-$9.47\%$ in private face and license plate detection.
An exception is DMC-Net, whose latency is seven times higher due to its expensive optical‑flow estimation and heavyweight ResNet‑152 backbone.
Although ComPrivDet incurs additional time from I-frame decoding, this cost is offset by inference reuse. Furthermore, leveraging rich pixel-domain data from decoded I-frames allows ComPrivDet to achieve much higher accuracy than existing frame-level methods—improving private face detection by $10.03\%$-$30.31\%$ and private license plate detection by $3.05\%$-$21.83\%$.

\subsubsection{Comparison with Existing GOP-Level Methods}

We compare ComPrivDet with two representative GOP‑level baselines: (1) SftRef-Net~\cite {tran2023fast}, which transfers I-frame features with RoIAlign and refines them using residuals; (2) MM-Net~\cite{wang2019fast}, which fuses multi-scale I-frame features and uses an LSTM for transfer.
As shown in Fig.~\ref{fig:cmp_comp_GOPLv}, ComPrivDet surpasses them by $7.17\%$-$8.28\%$ in private face detection accuracy and $0.61\%$-$0.83\%$ in private license plate detection accuracy.
We attribute this improvement to ComPrivDet’s YOLOv5 backbone, which outperforms MM-Net's ResNet-101 and SftRef-Net's SSD backbones. 
The smaller gains in license plate detection stem from plates covering limited image areas, making them hard to detect for all methods.

In addition, ComPrivDet achieves much higher inference efficiency than GOP-level baselines. Its inference latency is only $30.20\%$-$31.98\%$ for private face detection and $26.53\%$-$28.59\%$ for private license plate detection.
This efficiency advantage comes from:
(1) ComPrivDet selectively reuses I-frame results to skip P/B-frame inference, while MM-Net and SftRef-Net still require explicit inference. 
(2) $\mathcal{F}^{P/B}_\mathcal{PO}$ relies solely on compressed-domain data, whereas MM-Net and SftRef-Net also incorporate transferred I-frame features.

\section{Conclusion}
In this paper, we present ComPrivDet, an efficient framework for privacy object detection in compressed video-the first to apply inference reuse at scale.
ComPrivDet partially decodes I-frames to extract critical GOP features for a pixel-domain detector.
During streaming, while incrementally updating motion vectors and residuals to track P/B‑frame changes.
It skips lightly changing P/B-frames and applies a lightweight model to the rest, greatly accelerating inference.
On a large-scaledataset, ComPrivDet achieves $1.36\%$ higher accuracy and $57.23\%$ lower latency than pixel‑domain detectors,
and $9.84\%$ higher accuracy with $75.95\%$ lower latency than compressed‑domain methods.
Our Future work will extend ComPrivDet to multi‑object scenarios, enabling flexible, mix‑and‑match privacy processing.



\bibliographystyle{IEEEbib}
\bibliography{icme2026references}

\end{document}